\documentclass{article}
\pdfoutput=1
\usepackage{nips15submit_e,times}
\usepackage{amsmath}
\usepackage{amssymb}
\usepackage{amsthm}
\usepackage{url}
\usepackage{subcaption}
\usepackage{graphicx}
\usepackage{microtype}
\usepackage{textcomp}
\usepackage{titling}
\usepackage[backref]{hyperref}
\newtheorem{theorem}{Theorem}

\newcommand{\N}{\mathcal{N}}
\newcommand{\I}{\mathcal{I}}

\newcommand{\tr}{\text{tr}}
\newcommand{\cov}{\text{cov}}
\newcommand{\eps}{\varepsilon}
\renewcommand{\v}[1]{\mathbf{#1}}

\nipsfinalcopy

\begin{document}

\title{Gaussian Process Random Fields}

\author{
David A. Moore and Stuart J. Russell\\
Computer Science Division\\
University of California, Berkeley\\
Berkeley, CA 94709\\
\texttt{\{dmoore, russell\}@cs.berkeley.edu} 
}

\maketitle

\begin{abstract}
  Gaussian processes have been successful in both supervised and
  unsupervised machine learning tasks, but their computational
  complexity has constrained practical applications. We introduce a
  new approximation for large-scale Gaussian processes, the Gaussian
  Process Random Field (GPRF), in which local GPs are coupled via
  pairwise potentials. The GPRF likelihood is a simple, tractable, and
  parallelizeable approximation to the full GP marginal likelihood,
  enabling latent variable modeling and hyperparameter selection on
  large datasets.  We demonstrate its effectiveness on synthetic
  spatial data as well as a real-world application to seismic event
  location.
\end{abstract}

\section{Introduction}

Many machine learning tasks can be framed as learning a function given
noisy information about its inputs and outputs. In
regression and classification, we are given inputs and asked to
predict the outputs; by contrast, in latent variable modeling we are
given a set of outputs and asked to reconstruct the inputs that could
have produced them. Gaussian processes (GPs) are a flexible class of
probability distributions on functions that allow us to approach
function-learning problems from an appealingly principled and clean
Bayesian perspective. Unfortunately, the time complexity of exact GP
inference is $O(n^3)$, where $n$ is the number of data points. This makes exact GP calculations infeasible for
real-world data sets with $n > 10000$. 

Many approximations have been
proposed to escape this limitation. One particularly simple
approximation is to partition the input space into smaller blocks,
replacing a single large GP with a multitude of local ones. This gains
tractability at the price of a potentially severe independence assumption. 

In this paper we relax the strong independence assumptions of
independent local GPs, proposing instead a Markov random field (MRF) of
local GPs, which we call a Gaussian Process Random Field (GPRF). A
GPRF couples local models via pairwise potentials that incorporate
covariance information. This yields a surrogate for the full GP
marginal likelihood that is simple to implement and 
can be tractably evaluated and optimized on
large datasets, while still enforcing a smooth covariance
structure. The task of approximating the marginal likelihood is
motivated by unsupervised applications such as the GP latent variable
model \cite{lawrence2004gaussian}, but examining the predictions made by our model also
yields a novel interpretation of the Bayesian Committee Machine \cite{tresp2000bayesian}.

We begin by reviewing GPs and MRFs, and
some existing approximation methods for large-scale GPs. In
Section~\ref{sec:gprf} we present the GPRF objective and examine its
properties as an approximation to the full GP marginal likelihood.  We
then evaluate it on synthetic data as well as an application to seismic event location. 
\vspace{-0.2cm}
\section{Background}
\vspace{-0.2cm}
\subsection{Gaussian processes}
\vspace{-0.2cm}
Gaussian processes \cite{rasmussen2006} are
distributions on real-valued functions. GPs are
parameterized by a mean function $\mu_\theta(\v{x})$, typically assumed
without loss of generality to be $\mu(\v{x})=0$, and a covariance function (sometimes called a kernel)
$k_\theta(\v{x}, \v{x}')$,  with hyperparameters $\theta$. A common choice is the squared
exponential covariance, $k_{SE}(\v{x}, \v{x}') =
\sigma^2_f\exp\left(-\frac{1}{2}\|\v{x}-\v{x}'\|^2 / \ell^2\right)$, with
hyperparameters $\sigma^2_f$  and $\ell$ specifying respectively a
prior variance and correlation lengthscale. 

We say that a random function $f(x)$ is Gaussian process distributed if, for
any $n$ input points $X$, the vector of function values $\v{f} = f(X)$ is
multivariate Gaussian, $\v{f} \sim \N(\v{0}, k_\theta(X, X)).$ In many applications we
have access only to noisy observations $\v{y} = \v{f} + \v{\eps}$ for some
noise process $\v{\eps}$. If the noise is iid Gaussian, i.e., $\v{\eps}\sim
\N(\v{0}, \sigma_n^2 \I)$, then the observations are themselves Gaussian, $\v{y} \sim \N(\v{0}, K_y)$, where $K_y = k_\theta(X, X) + \sigma^2_n\I.$

The most common application of GPs is to Bayesian regression
\cite{rasmussen2006}, in which we attempt to predict the
function values $\v{f}^*$ at test points $X^*$ via the conditional
distribution given the training
data, $p(\v{f}^* | \v{y}; X, X^*, \theta)$. Sometimes, however, we do not observe the training inputs $X$,
or we observe them only partially or noisily. This setting is known as the Gaussian Process
Latent Variable Model (GP-LVM)  \cite{lawrence2004gaussian}; it uses
GPs as a model for unsupervised
  learning and nonlinear dimensionality reduction. The GP-LVM
  setting typically involves multi-dimensional observations, $Y = (\v{y}^{(1)}, \ldots,
  \v{y}^{(D)})$, with each output dimension $\v{y}^{(d)}$ modeled as an
  independent Gaussian process. The input locations and/or hyperparameters are typically sought via maximization of the {\em marginal likelihood}
\begin{align}
\mathcal{L}(X, \theta) = \log p(Y ; X, \theta) &= \sum_{i=1}^D -\frac{1}{2}\log |K_y| - \frac{1}{2} \v{y}_i^T K_y^{-1} \v{y}_i + C\nonumber\\
&= -\frac{D}{2}\log |K_y| - \frac{1}{2}\tr(K_y^{-1} YY^T) + C,\label{eqn:mlik}
\end{align}
though some recent work \cite{titsias2010bayesian,
  damianou2014} attempts to recover an
approximate posterior on $X$ by maximizing a variational bound. Given a
differentiable covariance function, this
maximization is typically performed by gradient-based methods, 
although local maxima can be a significant concern as
the marginal likelihood is generally non-convex.

\subsection{Scalability and approximate inference}
\label{sec:approx}

\begin{figure}
\centering

\begin{subfigure}[t]{.30\textwidth}
                \includegraphics[width=\textwidth]{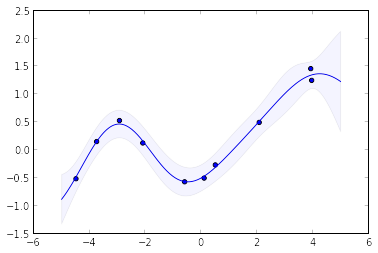}
                \caption{Full GP. }
                \label{fig:gp}
        \end{subfigure}\hspace{0.5em}
\begin{subfigure}[t]{.30\textwidth}
  \includegraphics[width=\textwidth]{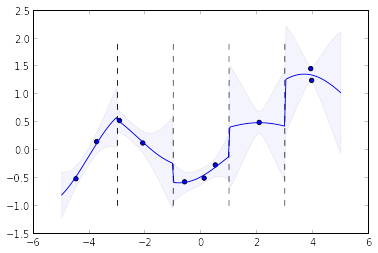}
  \caption{Local GPs.}
  \label{fig:local}
\end{subfigure}\hspace{0.5em}
\begin{subfigure}[t]{.30\textwidth}
  \includegraphics[width=\textwidth]{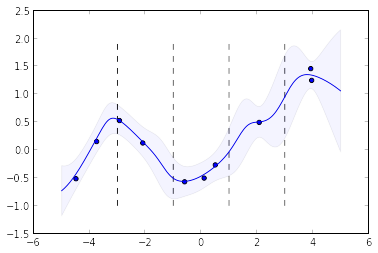}
  \caption{Bayesian committee machine.}
  \label{fig:bcm}
\end{subfigure}
\caption{Predictive distributions on a toy regression problem.}
\label{fig:approx}
\end{figure}

The main computational difficulty in GP methods is
the need to invert or factor the kernel matrix $K_y$, which requires time cubic
in $n$. In GP-LVM inference this must be done at every optimization step to evaluate
(\ref{eqn:mlik}) and its derivatives.

This complexity has inspired
a number of approximations. The most commonly studied are {\em inducing-point} methods, in which the unknown function is
represented by its values at a set of $m$ inducing
points, where $m \ll n$. These points can be chosen by maximizing the marginal
likelihood in a surrogate model \cite{quinonero2005, lawrence2007learning} or by
minimizing the KL divergence between the approximate and exact GP posteriors
\cite{titsias2009variational}. Inference in such models can
typically be done in $O(nm^2)$ time, but this comes at the price
of reduced representational capacity: while smooth functions with long lengthscales may
be compactly represented by a small number of inducing points, for
quickly-varying functions with significant local
structure it may be difficult to find any faithful representation more
compact than the complete set of training observations.

A separate class of approximations, so-called ``local'' GP methods
\cite{rasmussen2006, nguyen2009model, park2011domain}, involves
partitioning the inputs into blocks of $m$ points each, then modeling
each block with an independent Gaussian process. If the partition is
spatially local, this corresponds to a covariance function that imposes independence
between function values in different regions of the input
space. Computationally, each block requires only $O(m^3)$ time; the
total time is linear in the number of blocks. Local approximations
preserve short-lengthscale structure within each block, but their harsh independence assumptions can lead to
predictive discontinuities and inaccurate uncertainties
(Figure~\ref{fig:local}). These assumptions are problematic for
GP-LVM inference because the marginal likelihood becomes
discontinuous at block boundaries. Nonetheless, local GPs sometimes
work very well in practice, achieving results comparable to more
sophisticated methods in a fraction of the time \cite{chalupka2012}.

The Bayesian Committee Machine (BCM) \cite{tresp2000bayesian} attempts
to improve on independent local GPs by averaging the predictions of
multiple GP experts. The model is formally equivalent to an inducing-point model in which the {\em test points} are the inducing points,
i.e., it assumes that the training blocks are conditionally
independent given the test data. The BCM can yield high-quality
predictions that avoid the pitfalls of local GPs
(Figure~\ref{fig:bcm}), while
maintaining scalability to very large datasets
\cite{deisenroth2015distributed}. However, as a purely predictive
approximation, it is unhelpful in the GP-LVM setting, where we are
interested in the likelihood of our training set irrespective of any
particular test data.  The desire for a BCM-style approximation to the
marginal likelihood was part of the motivation for this current work;
in Section~\ref{sec:approx-predict} we show that the GPRF proposed in this
paper can be viewed as such a model.

Mixture-of-experts models \cite{rasmussen2002infinite, nguyen2014fast}
extend the local GP concept in a different direction: instead of
deterministically assigning points to GP models based on their spatial
locations, they treat the assignments as unobserved random variables
and do inference over them. This allows the model to adapt to different
functional characteristics in different regions of the space, at the
price of a more difficult inference task. We are not aware of
mixture-of-experts models being applied in the GP-LVM setting, though
this should in principle be possible.

Simple building blocks are often combined to create more complex
approximations. 
The PIC approximation \cite{snelson2007} blends a global
inducing-point model with local block-diagonal covariances, 
thus capturing a mix of global and local
structure, though with the same boundary discontinuities as in
``vanilla'' local GPs. A related approach is the use of covariance
functions with compact support \cite{vanhatalo2008} to capture local
variation in concert with global inducing points. \cite{chalupka2012} surveys and compares
several approximate GP regression methods on synthetic and real-world
datasets. 

Finally, we note here the similar title of \cite{zhong2010gaussian},
which is in fact orthogonal to the present work: they use a random
field as a {\em prior} on input locations, whereas this paper defines a
random field decomposition of the GP model itself, which may be
combined with any prior on $X$.

\subsection{Markov Random Fields}

We recall some basic theory regarding Markov random fields
(MRFs), also known as undirected graphical models \cite{koller2009probabilistic}. A pairwise
MRF consists of an undirected graph $(V, E)$, along with {\em node potentials} $\psi_i$ and {\em edge
potentials} $\psi_{ij}$, which define an {\em energy function} on a random vector $\v{y}$,
\begin{equation}
E(\v{y}) = \sum_{i\in V} \psi_{i}(\v{y}_i) + \sum_{(i,j)\in E}
\psi_{ij}(\v{y}_i, \v{y}_j),\label{eqn:mrf}
\end{equation}
where $\v{y}$ is partitioned into components
$\v{y}_i$ identified with nodes in the graph. This energy in turn defines a
probability density, the ``Gibbs
distribution'', given by $p(\v{y}) = \frac{1}{Z}\exp(-E(\v{y}))$ where
$Z = \int \exp(-E(\v{z})) d\v{z}$ is a normalizing constant.

Gaussian random fields are the special case of pairwise MRFs in which
the Gibbs distribution is multivariate Gaussian. Given a partition of
$\v{y}$ into sub-vectors $\v{y}_1, \v{y}_2, \ldots, \v{y}_M$, a
zero-mean Gaussian distribution with covariance $K$ and precision
matrix $J = K^{-1}$ can
be expressed by potentials 
\begin{equation}
\psi_i(\v{y}_i) = -\frac{1}{2}\v{y}_i^T
J_{ii} \v{y}_i, \qquad\psi_{ij}(\v{y}_i, \v{y}_j) = -\v{y}_i^T J_{ij}
\v{y}_j \label{eqn:gaussian-mrf}
\end{equation} where $J_{ij}$ is the submatrix of $J$ corresponding
to the sub-vectors $\v{y}_i$, $\v{y}_j$. The
normalizing constant $Z =
(2\pi)^{n/2}|K|^{1/2}$ involves the determinant of the covariance
matrix. Since edges whose potentials are zero can be dropped without
effect, the nonzero entries of the precision matrix can be seen as
specifying the edges present in the graph.

\section{Gaussian Process Random Fields}
\label{sec:gprf}

We consider a vector of $n$ real-valued\footnote{The extension to
  multiple independent outputs is straightforward.} observations
$\v{y} \sim \N(\v{0}, K_y)$ modeled by a GP, where $K_y$ is implicitly
a function of input locations $X$ and hyperparameters $\theta$. Unless
otherwise specified, all probabilities $p(\v{y}_i), p(\v{y}_i,
\v{y}_j)$, etc., refer to marginals of this full GP. We would like to
perform gradient-based optimization on the marginal likelihood
(\ref{eqn:mlik}) with respect to $X$ and/or $\theta$, but suppose that
the cost of doing so directly is prohibitive.

In order to proceed, we assume a partition $\v{y} = (\v{y}_1, \v{y}_2,
\ldots, \v{y}_M)$ of the observations into $M$ blocks of size at most
$m$, with an implied corresponding partition $X = (X_1, X_2, \ldots,
X_M)$ of the (perhaps unobserved) inputs. The source of this partition
is not a focus of the current work; we might imagine that the blocks
correspond to spatially local clusters of input points, assuming that
we have noisy observations of the $X$ values or at least a reasonable
guess at an initialization. We let $K_{ij} = \cov_\theta(\v{y}_i,
\v{y}_j)$ denote the appropriate submatrix of $K_{y}$, and $J_{ij}$
denote the corresponding submatrix of the precision matrix
$J_y=K_y^{-1}$; note that $J_{ij} \ne (K_{ij})^{-1}$ in general.

\subsection{The GPRF Objective}

Given the precision matrix $J_y$, we could use
(\ref{eqn:gaussian-mrf}) to represent the full GP distribution in
factored form as an MRF. This is not directly useful, since computing
$J_y$ requires cubic time. Instead we propose approximating the
marginal likelihood via a random field in which local GPs are
connected by pairwise potentials. Given an edge set which we will
initially take to be the complete graph, $E =\{(i,j) | 1\le i < j \le
M\}$, our approximate objective is
\begin{align}
q_{GPRF}(\v{y}; X, \theta)&= \prod_{i=1}^M p(\v{y}_i) 
\prod_{(i,j)\in E} \frac{p(\v{y}_i, \v{y}_j)}{p(\v{y}_i) p(\v{y}_j)},\label{eqn:gprf-naive}\\
&= \prod_{i=1}^M p(\v{y}_i)^{1-|E_i|} \prod_{(i,j)\in E} p(\v{y}_i, \v{y}_j) \nonumber
\end{align}
where $E_i$ denotes the neighbors of $i$ in the graph, and $p(\v{y}_i)$ and $p(\v{y}_{i},
\v{y}_j)$ are marginal probabilities under the full GP;
equivalently they are the likelihoods of local GPs
defined on the points $X_i$ and $X_i \cup X_j$ respectively. Note that
these local likelihoods depend implicitly on $X$ and $\theta$. Taking
the log, we obtain the energy function of an unnormalized MRF
\begin{equation}
\log q_{GPRF}(\v{y}; X, \theta) = \sum_{i=1}^M (1-|E_i|)\log p(\v{y}_i)
 + \sum_{(i,j)\in E} \log p(\v{y}_i, \v{y}_j)  \label{eqn:gprf-log}
\end{equation}
with potentials
\begin{equation}
\psi_i^{GPRF}(\v{y}_i) = (1-|E_i|)\log p(\v{y}_i), \qquad \psi_{ij}^{GPRF}(\v{y}_i, \v{y}_j) =
\log p(\v{y}_i, \v{y}_j).\end{equation}

We refer to the approximate objective (\ref{eqn:gprf-log}) as
$q_{GPRF}$ rather than $p_{GPRF}$ to emphasize that it is not in
general a normalized probability density. It can be interpreted as a
``Bethe-type'' approximation \cite{yedidia2001bethe}, in which a joint
density is approximated via overlapping pairwise marginals. In the
special case that the full precision matrix $J_y$ induces a tree
structure on the blocks of our partition, $q_{GPRF}$ recovers the
exact marginal likelihood. (This is shown in the supplementary material.) In general this will not be the case, but in the spirit of loopy
belief propagation \cite{murphy1999loopy}, we consider the
tree-structured case as an approximation for the general setting.

Before further analyzing the nature of the approximation, we first
observe that as a sum of local Gaussian log-densities, the objective (\ref{eqn:gprf-log})
is straightforward to implement and fast to evaluate. Each of the $O(M^2)$
pairwise densities requires $O((2m)^3) = O(m^3)$ time, for an overall complexity of
$O(M^2m^3) = O(n^2m)$ when $M=n/m$. The quadratic dependence on $n$ cannot
be avoided by any algorithm that computes similarities between all
pairs of training points; however, in practice we will consider ``local''
modifications in which $E$ is something smaller than all
pairs of blocks. For example, if each block is connected only to a
fixed number of spatial neighbors, the complexity reduces
to $O(nm^2)$, i.e., linear in $n$. In the special case where $E$ is
the empty set, we recover the exact likelihood of independent local GPs.

It is also straightforward to obtain the gradient of
(\ref{eqn:gprf-log}) with respect to hyperparameters $\theta$ and inputs $X$, by summing
the gradients of the local densities. The likelihood and gradient for each term in the sum
can be evaluated independently using only local subsets of the
training data, enabling a simple parallel implementation. 

Having seen that $q_{GPRF}$ can be optimized efficiently, it remains
for us to argue its validity as a proxy for the full GP marginal
likelihood. Due to space constraints we
defer proofs to the supplementary material, though our results are not
difficult. We first show that, like the full marginal likelihood (\ref{eqn:mlik}),
$q_{GPRF}$ has the form of a Gaussian distribution, but with a
different precision matrix.

\begin{theorem}
The objective $q_{GPRF}$ has the form of an unnormalized Gaussian density
with precision matrix $\tilde{J}$, with blocks $\tilde{J}_{ij}$ given by
\begin{equation}
\tilde{J}_{ii} = K_{ii}^{-1} + \sum_{j\in E_i} \left(Q^{(ij)}_{11}
    - K_{ii}^{-1}\right), \qquad \tilde{J}_{ij} =
  \left\{\begin{array}{ll}Q^{(ij)}_{12} & \text{ if } (i,j) \in E\\0
    & \text{ otherwise.}\end{array}\right),\label{eqn:approx-precision}
\end{equation}
where  $Q^{(ij)}$ is the {\em local precision  matrix} $Q^{(ij)}$
defined as the inverse of the marginal covariance,
\[Q^{(ij)} = \left(\begin{array}{cc} Q^{(ij)}_{11} &  Q^{(ij)}_{12}\\
  Q^{(ij)}_{21}  & Q^{(ij)}_{22}\end{array}\right) = \left(\begin{array}{cc} K_{ii} &  K_{ij}\\
  K_{ji}  & K_{jj}\end{array}\right)^{-1}.\]
\end{theorem}

Although the Gaussian density represented by $q_{GPRF}$ is not in general normalized, we show that it is {\em
  approximately} normalized in a certain sense. 

\begin{theorem}
The objective $q_{GPRF}$ is approximately normalized in the sense that
the optimal value of the {\em Bethe free energy} \cite{yedidia2001bethe}, 
\begin{equation}
F_B(b) = \sum_{i\in V} \left(\int_{\v{y}_i} b_i(\v{y}_i) \frac{(1-|E_i|)\ln
    b_i(\v{y}_i)}{\ln \psi_i(\v{y}_i)}\right)  + \sum_{(i,j)\in E} \left(\int_{\v{y}_i, \v{y}_j} b_{ij}(\v{y}_i,
  \v{y}_j) \ln \frac{b_{ij}(\v{y}_i,
  \v{y}_j)}{\psi_{ij}(\v{y}_i, \v{y}_j))}\right)
\label{eqn:bethe-energy} \approx \log Z,
\end{equation}
the approximation to the normalizing constant found by loopy belief
propagation, is precisely zero. Furthermore, this
optimum is obtained when the pseudomarginals $b_i, b_{ij}$ are
taken to be the true GP marginals $p_i, p_{ij}$. 
\end{theorem}

This implies that loopy belief propagation run on a GPRF
would recover the marginals of the true GP.

\subsection{Predictive equivalence to the BCM}
\label{sec:approx-predict}

We have introduced $q_{GPRF}$ as a surrogate model
for the training set $(X, \v{y})$; however, it is natural to extend
the GPRF to make predictions at a set of test points $X^*$, by including the
function values $\v{f}^* = f(X^*)$ as an $M+1$st block, with an edge to each of the training blocks. The resulting
predictive distribution,
\begin{align}
p_{GPRF}(\v{f}^* | \v{y}) \propto q_{GPRF}(\v{f}^*, \v{y}) 
&= p(\v{f}^*) \prod_{i=1}^M \frac{p(\v{y}_i,
  \v{f}^*)}{p(\v{y}_i) p(\v{f}^*)} \left(\prod_{i=1}^M p(\v{y}_i) \prod_{(i,j)\in E} \frac{p(\v{y}_i, \v{y}_j)}{p(\v{y}_i)
    p(\v{y}_j)}\right)  \nonumber \\
&\propto p(\v{f}^*)^{1-M} \prod_{i=1}^M p(\v{f}^* | \v{y}_i),
\end{align}
corresponds exactly to the prediction of 
the Bayesian Committee Machine (BCM) \cite{tresp2000bayesian}. This motivates the
GPRF as a natural extension of the BCM as a model for the training
set, providing an alternative to the standard transductive
interpretation of the BCM.\footnote{The GPRF is still transductive, in
  the sense that adding a test block $\v{f^*}$ will change the
  marginal distribution on the training observations $\v{y}$, as 
  can be seen explicitly in the precision matrix (\ref{eqn:approx-precision}). The contribution of the GPRF is that it provides a reasonable model for
  the training-set likelihood even in the absence of test
  points. } A similar derivation shows that the conditional distribution of any
block $\v{y}_i$ given all other blocks $\v{y}_{j\ne i}$ also takes the
form of a BCM prediction, suggesting the possibility of
{\em pseudolikelihood} training \cite{besag1975statistical}, i.e.,
directly optimizing the quality of BCM predictions on held-out blocks
(not explored in this paper).

\section{Experiments}
\vspace{-0.2cm}
\subsection{Uniform Input Distribution}
\vspace{-0.1cm}
We first consider a 2D synthetic dataset intended to simulate spatial location tasks such as WiFi-SLAM
\cite{ferris2007wifi} or seismic event location (below), in which we
observe high-dimensional measurements but have only noisy information
regarding the locations at which those measurements were
taken. We sample $n$ points uniformly from the square of side length
$\sqrt{n}$ to generate the true inputs $X$, then sample 50-dimensional output $Y$ from independent GPs
with SE kernel $k(r) = \exp(-(r/\ell)^2)$ for $\ell=6.0$ and noise standard deviation
$\sigma_n = 0.1$. The {\em observed} points
$X^\text{obs} \sim N(X, \sigma^2_\text{obs}I)$ arise by corrupting $X$ with
isotropic Gaussian noise of standard deviation
$\sigma_\text{obs}=2$. The parameters $\ell$, $\sigma_n$, and
$\sigma_\text{obs}$ were chosen to generate problems with
interesting short-lengthscale structure for which
GP-LVM optimization could nontrivially improve the initial noisy locations. Figure \ref{fig:synthX} shows a typical sample from this model.

\begin{figure}
 \centering
 \begin{subfigure}[t]{.215\textwidth}
         \includegraphics[width=\textwidth]{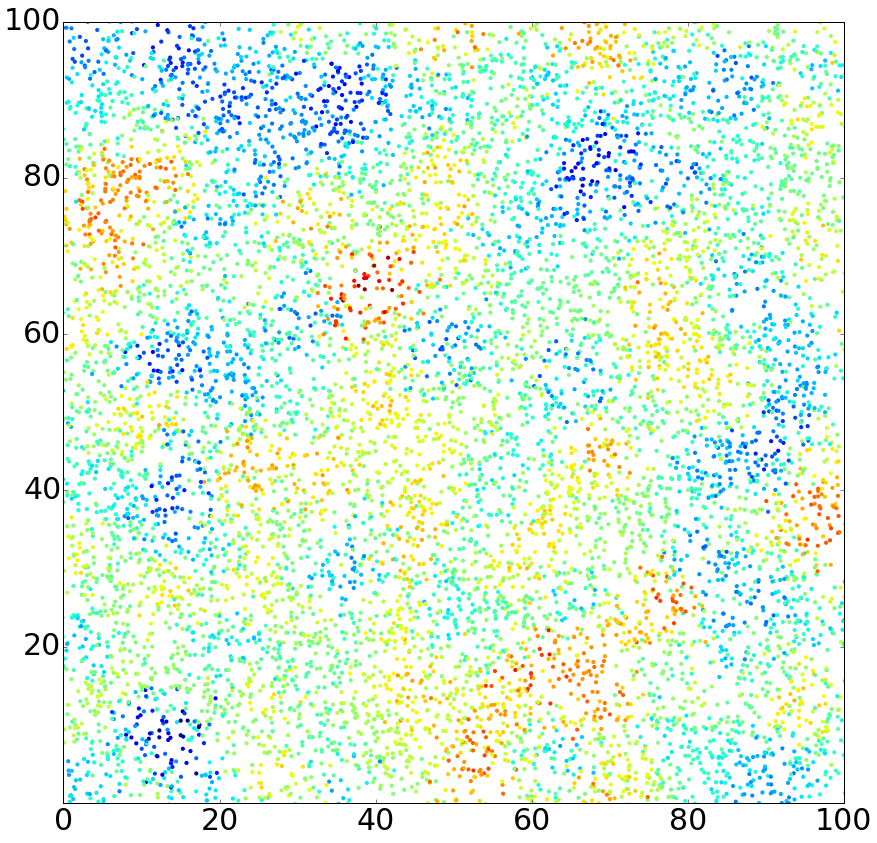}
         \caption{Noisy observed locations: mean error 2.48.}
         \label{fig:synthX}
     \end{subfigure}\hspace{0.5em}  
\begin{subfigure}[t]{.21\textwidth}
        \includegraphics[width=\textwidth]{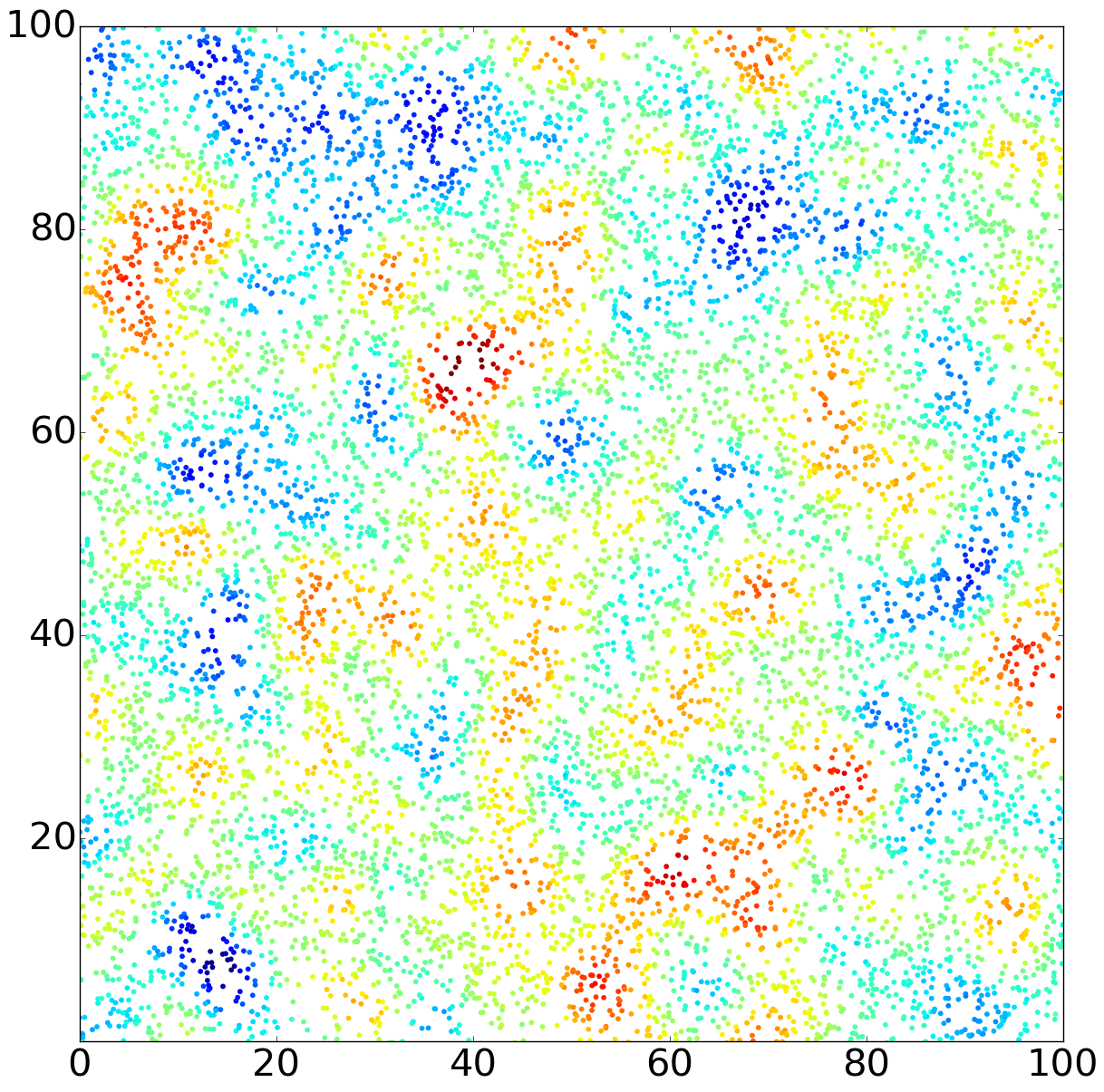}
        \caption{Full GP: 0.21.}
        \label{fig:synth_exact} 
    \end{subfigure}\hspace{0.5em}
\begin{subfigure}[t]{.215\textwidth}
        \includegraphics[width=\textwidth]{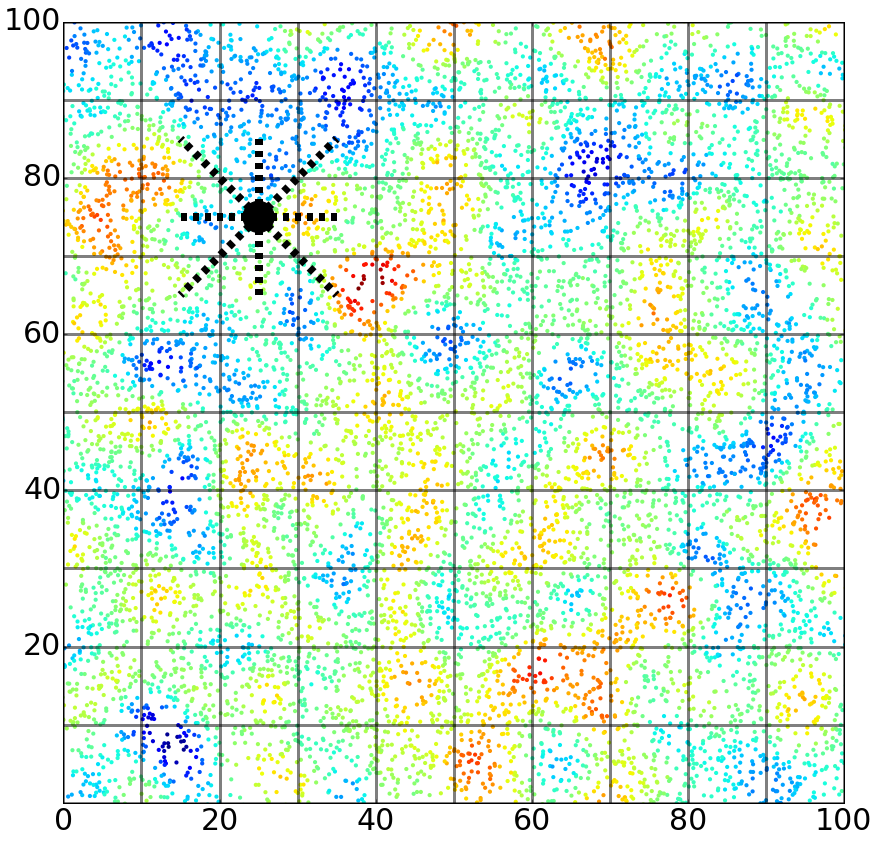}
        \caption{GPRF-100: 0.36. (showing grid cells)}
        \label{fig:synth_gprf} 
    \end{subfigure}\hspace{0.5em}
\begin{subfigure}[t]{.21\textwidth}
        \centering
        \includegraphics[width=\textwidth]{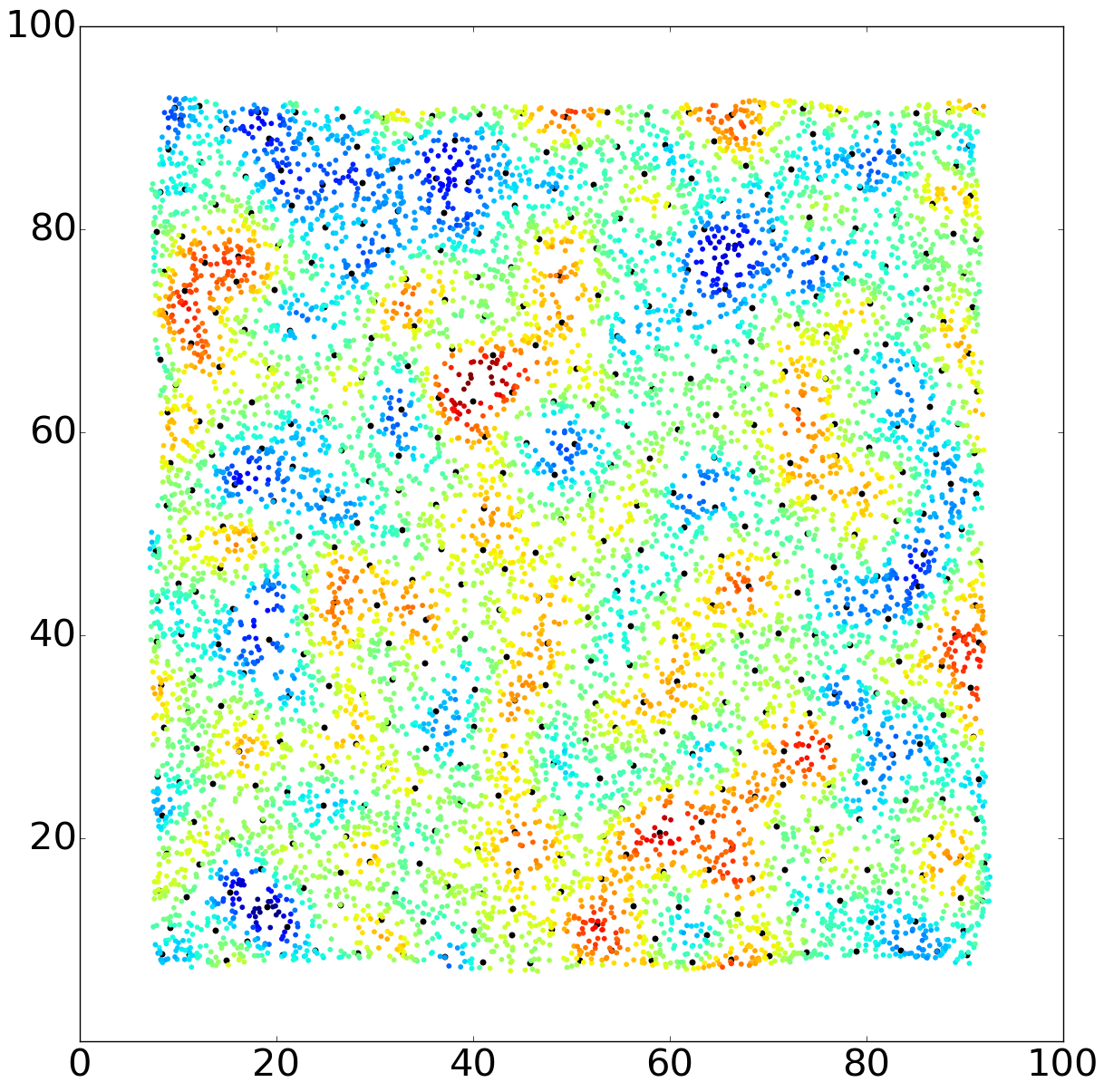}
        \caption{FITC-500: 4.86. (with inducing points, note contraction)}
        \label{fig:synth_fitc500} 
\end{subfigure}
\caption{Inferred locations on synthetic data
  ($n=10000$), colored by the first output dimension $\v{y}_1$. }
\label{fig:synth_data}

\end{figure}


For local GPs and GPRFs, we take the spatial partition to be a
grid with $n/m$ cells, where $m$ is the desired number of points per
cell.  The GPRF edge set $E$ connects each cell to its 
eight neighbors (Figure
\ref{fig:synth_gprf}), yielding linear time
complexity $O(nm^2)$. During optimization, a practical choice
is necessary: do we use a fixed partition of the points, or re-assign
points to cells as they cross spatial boundaries? The latter
corresponds to a coherent (block-diagonal) spatial covariance function, but introduces
discontinuities to the marginal likelihood. In our experiments the GPRF was not sensitive to
this choice, but local GPs performed more reliably with fixed
spatial boundaries (in spite of the discontinuities), so we used this
approach for all experiments.
\begin{figure}
\begin{subfigure}[t]{.32\textwidth}
                 \includegraphics[width=\textwidth, trim=0cm -0.0cm 0cm 0 -0.0cm, clip]{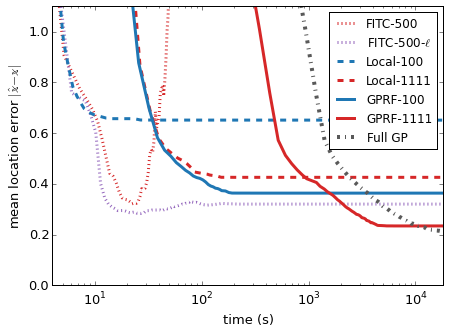}
                 \caption{Mean location error over time for $n=10000$,
                   including comparison to full GP.}
                 \label{fig:truegp}
\end{subfigure}\hspace{0.3em}
\begin{subfigure}[t]{.33\textwidth}
                 \includegraphics[width=\textwidth,  trim=0cm -0.2cm 0cm 0 -0.0cm, clip]{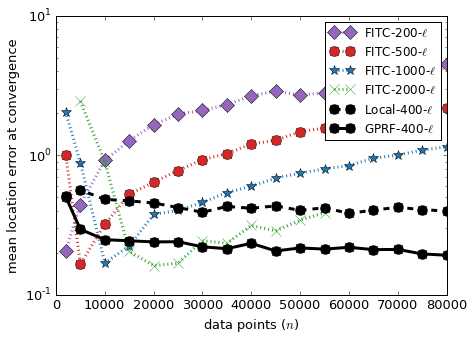}
                 \caption{Mean error at convergence as a
                   function of $n$, with learned
                   lengthscale.}
                 \label{fig:fitc_lscale}
         \end{subfigure}\hspace{0.3em}
 \begin{subfigure}[t]{.32\textwidth}
                 \includegraphics[width=\textwidth, trim=0 0 0 0 0cm, clip]{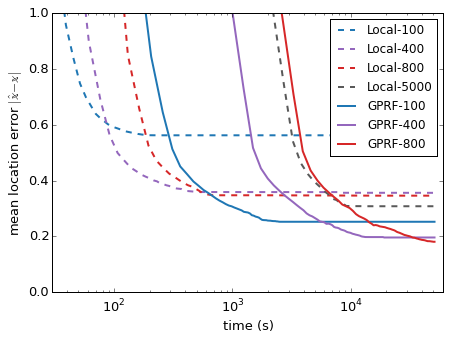}
                 \caption{Mean location error over time for $n=80000$.}
                 \label{fig:eighty}
         \end{subfigure}
\caption{Results on synthetic data.}
\label{fig:synthresults}
\vspace{-0.5cm}
\end{figure} 

For comparison, we also evaluate the Sparse GP-LVM, implemented in GPy
\cite{gpy2014}, which uses the FITC approximation to the
marginal likelihood \cite{lawrence2007learning}. (We also considered
the Bayesian GP-LVM \cite{titsias2010bayesian}, but found it to be more
resource-intensive with no meaningful difference in results on this
problem.) Here the approximation parameter $m$ is the number of inducing points. 

We ran L-BFGS optimization to recover maximum a posteriori (MAP)
locations, or local optima thereof. Figure \ref{fig:truegp} shows
mean location error (Euclidean distance) for $n=10000$ points; at this
size it is tractable to compare directly to the full GP-LVM. The GPRF
with a large block size ($m$=1111, corresponding to a 3x3 grid) nearly matches the solution
quality of the full GP while requiring less time, while the local
methods are quite fast to converge but become stuck at inferior
optima. The FITC optimization exhibits an interesting pathology: it
initially moves towards a good solution but then diverges towards what
turns out to correspond to
a contraction of the space (Figure \ref{fig:synth_fitc500}); we
conjecture this is because there are not enough inducing points to
faithfully represent the full GP distribution over the entire space. A
partial fix is to allow FITC to jointly optimize over locations and
the correlation lengthscale $\ell$; this yielded a biased lengthscale
estimate $\hat{\ell} \approx 7.6$ but more accurate
locations (FITC-500-$\ell$ in Figure \ref{fig:truegp}).

To evaluate scaling behavior, we next considered problems of
increasing size up to $n=80000.$\footnote{The astute reader will wonder how we generated
  synthetic data on problems that are clearly too large for an exact
  GP. For these synthetic problems as well as the seismic example below, the
  covariance matrix is relatively sparse, with only \texttildelow 2\% of entries
  corresponding to points within six kernel lengthscales of each other. By considering only these
  entries, we were able to draw samples using a sparse Cholesky
  factorization, although this required approximately 30GB of RAM. Unfortunately, this approach does not straightforwardly
  extend to GP-LVM inference under the exact GP, as the standard
  expression for the marginal likelihood
  derivatives \[\frac{\partial}{\partial \v{x}_{i}} \log p(\v{y})
  = \frac{1}{2} \tr\left( \left((K_y^{-1} \v{y}) (K_y^{-1} \v{y})^T -
      K_y^{-1}\right) \frac{\partial K_y}{\partial \v{x}_{i}} \right)\] involves the full
  precision matrix $K_y^{-1}$ which is not sparse in general. Bypassing this expression via 
  automatic differentiation through the sparse Cholesky
  decomposition could perhaps allow exact GP-LVM inference to scale to
  somewhat larger problems.} Out of generosity to FITC we allowed each method to learn
its own preferred lengthscale. Figure \ref{fig:fitc_lscale}
reports the solution quality at convergence, showing that even with an
adaptive lengthscale, FITC requires increasingly many inducing points
to compete in large spatial domains. This is intractable for larger
problems due to $O(m^3)$ scaling; indeed, attempts to run at $n>55000$
with 2000 inducing points exceeded 32GB of available memory. Recently,
more sophisticated inducing-point methods have claimed scalability to very large
datasets \cite{hensman2013gaussian, gal2014distributed}, but they do
so with $m\le 1000$; we expect that
they would hit the same fundamental scaling constraints for problems
that inherently require many inducing points. 

On our largest synthetic problem, $n=80000$, inducing-point
approximations are intractable, as is the full GP-LVM. Local GPs converge more quickly
than GPRFs of equal block size, but the GPRFs find higher-quality
solutions (Figure \ref{fig:eighty}). After a short initial period, the
best performance always belongs to a GPRF, and at the conclusion of 24 hours the best GPRF
solution achieves mean error 42\% lower than the best local
solution (0.18 vs 0.31).

\subsection{Seismic event location}
\vspace{-0.2cm}
 \begin{figure}
 \centering
  \begin{subfigure}[t]{.50\textwidth}
                  \includegraphics[width=\textwidth, trim=0cm -2cm 0 0cm, clip]{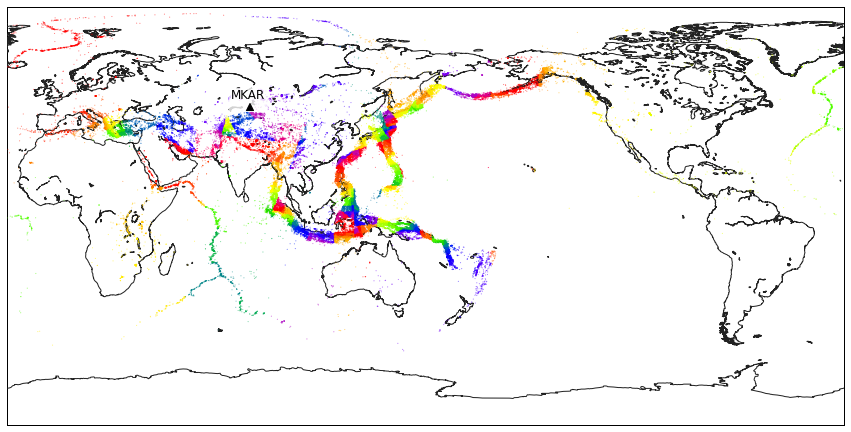}
                  \caption{Event map for seismic dataset.}
                  \label{fig:seismicX}
          \end{subfigure}\hspace{0.5em}
  \begin{subfigure}[t]{.45\textwidth}
                  \includegraphics[width=\textwidth]{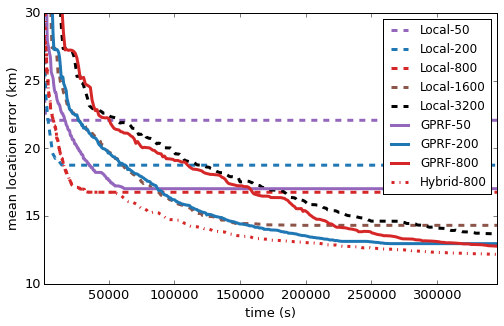}
                  \caption{Mean location error over time.}
                  \label{fig:seismicresults}
          \end{subfigure}
 \caption{Seismic event location task.}
 \label{fig:seismic}
\vspace{-0.5cm}
 \end{figure}

We next consider an application to seismic event location, which formed the
motivation for this work. Seismic waves can be
viewed as high-dimensional vectors generated from an
underlying three-dimension manifold, namely the Earth's
crust. Nearby events tend to generate similar waveforms; we
can model this spatial correlation as a Gaussian process. Prior information regarding the event
locations is available from traditional travel-time-based location systems
\cite{ISCcitation2015}, which produce an independent Gaussian uncertainty ellipse for each event.

A full probability model of seismic waveforms, accounting for background
noise and performing joint alignment of arrival times, is beyond the scope of this
paper. To focus specifically on the ability to approximate GP-LVM inference, we
used real event locations but generated synthetic waveforms  by sampling from a 50-output GP using a Mat\'ern kernel
\cite{rasmussen2006} with $\nu=3/2$ and a lengthscale of 40km. We also
generated observed location estimates $X^\text{obs}$, by corrupting the
true locations with Gaussian noise of standard deviation
20km in each dimension. Given the
observed waveforms and noisy locations, we are interested in recovering the
latitude, longitude, and depth of each event.

Our dataset consists of 107556 events detected at the Mankachi array
station in Kazakstan between 2004 and 2012. Figure~\ref{fig:seismicX}
shows the event locations, colored to reflect a principle axis tree
partition \cite{mcnames2001fast} into blocks of $400$ points (tree
construction time was negligible). The GPRF edge set contains all pairs of
blocks for which any two points had initial locations within one
kernel lengthscale (40km) of each other. We also evaluated
longer-distance connections, but found that this relatively local
edge set had the best performance/time tradeoffs: eliminating edges not
only speeds up each optimization step, but in some cases actually
yielded faster per-step convergence (perhaps because denser edge sets
tended to create large cliques for which the pairwise GPRF objective is
a poor approximation).

Figure \ref{fig:seismicresults} shows the quality of recovered locations as a
function of computation time; we jointly optimized over event locations as
well as two lengthscale parameters (surface distance
and depth) and the noise variance $\sigma^2_n$. Local GPs 
perform quite well on this task, but the best
GPRF achieves 7\% lower mean error than the best local
GP model (12.8km vs 13.7km, respectively) given equal time. An even
better result can be obtained by using the results of a local GP
optimization to initialize a GPRF. Using the same
partition ($m=800$) for both local GPs and the GPRF, this ``hybrid'' method
gives the lowest final error (12.2km), and is dominant across a wide
range of wall clock times, suggesting it as a promising
practical approach for large GP-LVM optimizations. 
 
 \vspace{-0.2cm}
\section{Conclusions and Future Work}
\vspace{-0.2cm}
The Gaussian process random field is a tractable and effective
surrogate for the GP marginal likelihood. It has the flavor of
approximate inference methods such as loopy belief propagation, but
can be analyzed precisely in terms of a deterministic approximation to
the inverse covariance, and provides a new training-time
interpretation of the Bayesian Committee Machine. It is easy to
implement and can be straightforwardly parallelized. 

One direction for future work involves finding partitions for
which a GPRF performs well, e.g., partitions that induce a block near-tree
structure. A perhaps related question is identifying when the GPRF
objective defines a normalizable probability
distribution (beyond the case of an exact tree structure) and under
what circumstances it is a good approximation to the exact GP likelihood.

This evaluation in this paper focuses on spatial data; however, both local
GPs and the BCM have been successfully applied to high-dimensional
regression problems \cite{chalupka2012, deisenroth2015distributed}, so
exploring the effectiveness of the GPRF for dimensionality reduction
tasks would also be interesting. Another useful avenue is to integrate the
GPRF framework with other approximations: since the GPRF and inducing-point methods have complementary strengths -- the GPRF is useful for
modeling a function over a large space, while inducing points are useful
when the density of available data in some region of the space
exceeds what is necessary to represent the function -- an integrated
method might enable new applications for which neither approach alone
would be sufficient. 

\vspace{-0.2cm}
\subsubsection*{Acknowledgements}
\vspace{-0.1cm}
We thank the anonymous reviewers for their helpful 
suggestions.  This work was supported by DTRA grant \#HDTRA-11110026,
and by computing resources donated by Microsoft Research under an Azure for Research grant.

\bibliographystyle{unsrt}
\bibliography{refs}

\appendix
\title{Gaussian Process Random Fields: Supplementary Material}

\author{
David A. Moore and Stuart J. Russell\\
}

\maketitle
 
This file contains additional derivations for our NIPS 2015 paper, ``Gaussian Process Random Fields''. Notation used here follows the notation of the paper. Code to construct the datasets and reproduce the experimental results is available online at \url{https://github.com/davmre/gprf/}.

\section{Block tree structure}

It is straightforward to see that the GPRF objective is exact when the MRF induced by the true precision matrix $J$, with respect to our chosen partition of $\v{y}$, is a tree. For any choice of root node $\v{y}_\text{root}$, the tree structure implies that we can write the true GP distribution as a product of parent-conditional distributions,
\[p(\v{y}) = p(\v{y}_\text{root}) \prod_{i \ne \text{root}} p(\v{y}_i | \v{y}_{\pi(i)})\]
where $\pi(i)$ is the (unique) parent of node $i$ with respect to our chosen root. Then expanding the conditional distribution
\begin{align*}
p(\v{y}) &= p(\v{y}_\text{root}) \prod_{i \ne \text{root}} \frac{p(\v{y}_i , \v{y}_{\pi(i)})}{p(\v{y}_{\pi(i)})}\\
 &= p(\v{y}_\text{root}) \prod_{i \ne \text{root}} p(\v{y}_i) \frac{p(\v{y}_i , \v{y}_{\pi(i)})}{p(\v{y}_i) p(\v{y}_{\pi(i)}) }\\
 &= \left(\prod_{i} p(\v{y}_i) \right) \left(\prod_{i \ne \text{root}} p(\v{y}_i) \frac{p(\v{y}_i , \v{y}_{\pi(i)})}{p(\v{y}_i) p(\v{y}_{\pi(i)}) }\right)
\end{align*}
yields exactly the GPRF objective for the edge set $E={(i, \pi(i)})$, i.e., the edges that define the tree. 

Note that the structure of the MRF induced by the true GP will depend on the partition we choose: a given precision matrix may induce a tree structure for some choices of partition but not for others (e.g., even a fully dense matrix can be viewed as a tree for trivial partitions that split the dataset into only one or two blocks). 

In many cases it is easier to reason about the structure of the covariance matrix than that of the precision matrix. Assuming a stationary kernel, nonzero (or non-negligible) entries of the covariance matrix correspond to data points that are nearby to each other, meaning that the sparsity pattern of the covariance matrix reflects the geometry of the data itself. If the data can be viewed as lying on a treelike manifold -- for example, seismic fault lines, or even trivial special cases such as time series data which lies on the real line -- then for reasonable choices of partition, a graph connecting nearby blocks of data points will have a tree structure. Of course, there is no formal guarantee that this structure will fully carry over to the precision matrix, though intuitively we'd expect that points very distant from each other are also unlikely to interact strongly in the precision matrix. 

\section{Approximation to the true Gaussian}
\label{sec:approx-gaussian}

In this section we prove Theorem 1 from the main text, showing that $q_{GPRF}$ is an unnormalized Gaussian density with a particular precision matrix. 

For any pair of blocks $(i,j)$, define the {\em local precision  matrix} $Q^{(ij)}$ to be the inverse of the marginal covariance,
\[Q^{(ij)} = \left(\begin{array}{cc} Q^{(ij)}_{11} &  Q^{(ij)}_{12}\\
  Q^{(ij)}_{21}  & Q^{(ij)}_{22}\end{array}\right) = \left(\begin{array}{cc} K_{ii} &  K_{ij}\\
  K_{ji}  & K_{jj}\end{array}\right)^{-1},\]
The notation $Q^{(ij)}$ is used to distinguish these local precision
matrices from the blocks $J_{ij}$ of the global precision matrix. Writing $q_{GPRF}$ in terms of unnormalized Gaussian densities,
\begin{align*}
\log q_{GPRF}(\v{y}) &= -\frac{1}{2} \sum_{i=1}^M (1-|E_i|) \v{y}_i^T
K_{ii}^{-1} \v{y}_i -\frac{1}{2}  \sum_{(i,j)\in E} \left(\begin{array}{c}
      \v{y}_i \\ \v{y}_j\end{array}\right)^T Q^{(ij)}\left(\begin{array}{c}
      \v{y}_i \\
      \v{y}_j\end{array}\right) + C\\
&= -\frac{1}{2}\sum_{i=1}^M  \v{y}_i^T \left(K_{ii}^{-1} - |E_i|
  K_{ii}^{-1}\right)\v{y}_i -\frac{1}{2}  \left(\sum_{(i,j)\in E} \v{y}_i^T
Q^{(ij)}_{11} \v{y}_i + 2\v{y}_i^T Q^{(ij)}_{12}\v{y}_j + \v{y}_j^T Q^{(ij)}_{22}\v{y}_j\right) + C\\
&= -\frac{1}{2}\sum_{i=1}^M  \v{y}_i^T \left(K_{ii}^{-1} + \sum_{j\in E_i}
\left(Q^{(ij)}_{11} - K_{ii}^{-1}\right) \right)\v{y}_i - \sum_{(i,j)\in E}
\v{y}_i^T Q^{(ij)}_{12} \v{y}_j + C
\end{align*}
we obtain the standard form of a Gaussian
MRF (expression (3) from the main text) showing that $q_{GPRF}$ does in fact induce a Gaussian density on
$\v{y}$. Note that in passing from the second to the third line we used the fact that $Q^{ij}_{11} = Q^{ji}_{22}$, by definition. This Gaussian representation allows us to read off the implicit precision matrix $\tilde{J}$ in block wise form
\begin{equation}
\tilde{J}_{ii} = K_{ii}^{-1} + \sum_{j\in E_i} \left(Q^{(ij)}_{11}
    - K_{ii}^{-1}\right), \qquad \tilde{J}_{ij} =
  \left\{\begin{array}{ll}Q^{(ij)}_{12} & \text{ if } (i,j) \in E\\0
    & \text{ otherwise.}\end{array}\right.\label{eqn:approx-precision}
\end{equation}
We see that the off-diagonal blocks of the precision matrix are simply
the corresponding blocks of the pairwise local precisions. Each
diagonal block, by contrast, combines the inverse of the local
covariance matrix with corrections from the pairwise
precisions. Note that the approximate precision $\tilde{J}$ may not be positive definite. In this case $q_{GPRF}$ is not a normalizable density,
although it is still ``approximately normalized'' in the sense of the next section. 

\section{Approximate normalization}
\label{sec:approx-norm}

In this section we prove Theorem 2 from the main text:

\setcounter{theorem}{1}
\begin{theorem}
The objective $q_{GPRF}$ is approximately normalized in the sense that
the optimal value of the {\em Bethe free energy} \cite{yedidia2001bethe}, 
\begin{equation}
F_B(b) = \sum_{i\in V} \left(\int_{\v{y}_i} b_i(\v{y}_i) \frac{(1-|E_i|)\ln
    b(\v{y}_i)}{\ln \psi_i(\v{y}_i)}\right)  + \sum_{(i,j)\in E} \left(\int_{\v{y}_i, \v{y}_j} b_{ij}(\v{y}_i,
  \v{y}_j) \ln \frac{b_{ij}(\v{y}_i,
  \v{y}_j)}{\psi_{ij}(\v{y}_i, \v{y}_j))}\right)
\label{eqn:bethe-energy} \approx \log Z,
\end{equation}
the approximation to the normalizing constant found by loopy belief
propagation, is precisely zero. Furthermore, this
optimum is obtained when the pseudomarginals $b_i, b_{ij}$ are
taken to be the true GP marginals $p_i, p_{ij}$. 
\end{theorem}
\begin{proof}
These claims are established rather directly by substituting the GPRF potentials $\psi_i^{GPRF}, \psi_{ij}^{GPRF}$ for the log pseudomarginals $\log \psi_{i}, \log \psi_{ij}$ in (\ref{eqn:bethe-energy}), yielding
\begin{align}
 F_B(b)&= \sum_{i\in V} \left(\int_{\v{y}_i} b_i(\v{y}_i) \frac{(1-|E_i|)\ln
    b_i(\v{y}_i)}{(1-|E_i|) \ln p(\v{y}_i) }\right)  + \sum_{(i,j)\in E} \left(\int_{\v{y}_i, \v{y}_j} b_{ij}(\v{y}_i,
  \v{y}_j) \ln \frac{b_{ij}(\v{y}_i,
  \v{y}_j)}{p(\v{y}_i, \v{y}_j))}\right) \nonumber\\
&= \sum_{i\in V} KL[b_i \| p_i] + \sum_{(i,j)\in E} KL[b_{ij}\| p_{ij}],
\end{align}
where $KL[b\|p] = \int b(\v{x}) \ln \frac{b(\v{x})}{p(\v{x})}d\v{x}$
is the Kullback-Liebler divergence between distributions $b$ and
$p$.  This is minimized when the distributions are equal, at which
point the divergence is zero. Thus, taking $b_i=p_i$and
$b_{ij}=p_{ij}$ yields the optimal value $F_B=0$. 
\end{proof}

We might have hoped that, as local GPs match the marginal
distributions of the full GP on individual blocks, perhaps a
higher-order approximation could match the exact marginals on pairs of
blocks. This is not possible, since any Gaussian distribution whose
pairwise marginals match the full GP must in fact be the full GP
(Gaussians are entirely characterized by their covariances). Instead we
can view $q_{GPRF}$ as {\em approximately} matching the pairwise
marginals of the full GP, in the sense that the pseudomarginals found
by running loopy belief propagation on $q_{GPRF}$ are in fact the true
marginals of the full GP. This is a consequence of the fact that loopy
BP converges to stationary points of the Bethe energy
\cite{yedidia2001bethe}.

\end{document}